\documentclass[11pt]{amsart}%
\usepackage{amssymb,amsmath,amsfonts,latexsym,amsthm,geometry,graphicx}
\usepackage{amsmath}
\usepackage{amsfonts}
\usepackage{amssymb}
\usepackage{graphicx}
\usepackage{color}%
\usepackage{mathrsfs,amsmath}
\providecommand{\U}[1]{\protect\rule{.1in}{.1in}}
\providecommand{\U}[1]{\protect\rule{.1in}{.1in}}
\providecommand{\U}[1]{\protect\rule{.1in}{.1in}}
\providecommand{\U}[1]{\protect\rule{.1in}{.1in}}
\newtheorem{theorem}{Theorem}[section]
\newtheorem{corollary}[theorem]{Corollary}

\theoremstyle{definition}

\newtheorem{example}[theorem]{Example}

\newtheorem{remark}[theorem]{Remark}
\newtheorem{definition}[theorem]{Definition}
\begin{document}	
	
\title[Language learnability in the limit for general metrics: a Gold-Angluin result]{Language learnability in the limit for general metrics: a Gold-Angluin result}

\author[F. C. Alves]{Fernando C. Alves}
\thanks{Corresponding author's email: cabralalvesf@gmail.com \\ \indent Departamento de Matem\'{a}tica and Programa de Pós-graduação em Linguística (PROLING).\\ \indent Universidade Federal da Para\'{\i}ba \\ \indent 58.051-900 - Jo\~{a}o Pessoa, Brazil}

\subjclass[2010]{}

\begin{abstract}  
In his pioneering work in the field of Inductive Inference, Gold (1967) proved that a set containing all finite languages and at least one infinite language over the same fixed alphabet is not learnable in the exact sense. Within the same framework, Angluin (1980) provided a complete characterization for the learnability of language families. Mathematically, the concept of exact learning in that classical setting can be seen as the use of a particular type of metric for learning in the limit. In this short research note we use Niyogi's extended version of a theorem by Blum and Blum (1975) on the existence of locking data sets to prove a necessary condition for learnability in the limit of any family of languages in any given metric. This recovers Gold's theorem as a special case. Moreover, when the language family is further assumed to contain all finite languages, the same condition also becomes sufficient for learnability in the limit. 
\end{abstract}
\maketitle
\section{Introduction}
In his pioneering work in the field of Inductive Inference, Gold (1967, \cite{gold}) proved that a collection containing all finite languages and at least one infinite language, over the same fixed alphabet, is not learnable by any algorithm in the exact sense (see Definition \ref{exact learning} below). Within the same framework, Angluin (1980) provided a complete characterization for the learnability of languages. Niyogi (2006, \cite{niyogy}) emphasizes how the concept of locking data sets developed in Blum and Blum (1975, \cite{blumandblum}) holds the key to both aforementioned results. Mathematically, the concept of exact learning within the inductive inference tradition can be seen as the use of a specific type of metric in the definition of learning in the limit. In this short note we use Niyogi's $\varepsilon$-version of a theorem by Blum and Blum on the existence of locking data sets and extend his arguments to prove a necessary condition for learnability in the limit for any family of languages in any given metric. This recovers Gold's classical theorem as a special case. Moreover, when the language family is further assumed to contain all finite languages (such as in Gold's theorem), the same condition becomes also sufficient for learnability in the limit for an arbitrary metric.  Finally, inspired by these results, we exhibit a simple metric for which a family as described in Gold's theorem can be learned but observe that such counting metrics are of no relevance for cognitive models of acquisition. As a final introductory remark, in the event of failing to make it clear elsewhere, let it be registered here that any new result in this note found relevant by the reader should also be credited to P. Niyogi. We now turn to the detailed discussion of what has just been outlined.

\section{Exact learnability in inductive inference}

Let $A$ be a non-empty countable set of symbols we call an alphabet. If $n$ is a natural number, we denote by $A^{n}$ the cartesian product $A \times \cdots \times A$ (with $n$ copies of $A$) and, as usual, define a (formal) language over the alphabet $A$ to be any subset of the universe $\sum^{*}:=\cup_{n=1}^{\infty}A^{n}$. We are only interested in recursively enumerable languages ($r.e.$ languages), that is, those that can be recognized by Turing machines or, equivalenty, which are defined by a grammar of type $0$ in the Chomsky hierarchy (see \cite{chomsky}, \cite{chiswell}). Since both the set of all r.e. languages and languages themselves are countable, we can freely enumerate languages and their elements (usually called words in this general context). We use the terms 'language' and 'grammar' interchangeably (if two or more grammars generate the same language, one can think of them as an equivalence class here).      

The set of all possible data sets is given by $D=\bigcup_{n=1}^{\infty}(\sum^{*})^{n}$ and we denote by $\mathscr{L}$ the set of all $r. e.$ languages. Henceforth, any mention of a collection of languages should be understood to be defined over the same fixed alphabet and to be a subset of $\mathscr{L}$. A learning algorithm $\mathcal{A}$ is then defined to be an effective procedure from $D$ to $\mathscr{L}$. The general problem of language acquisition/learning faced by $\mathcal{A}$ is that of converging (in some specified sense) to a target language as more data are received. There is no universally adopted definition of convergence for learning algorithms and that is one of the main points in which learning theories differ. The notion we discuss here stems from the pioneering non-probabilistic tradition of Gold in Inductive Inference (\cite{gold}). Before we turn to it, however, we need to introduce the notion of a text.

\begin{definition}
A text for a language $L$ is a surjective sequence $(t(k))_{k=1}^{\infty}$ in $L$. An important notational distinction is in order: here we denote the $k$-th element of a text by $t(k)$ and the symbol $t_{k}$ is used to symbolize the $k-$tuple $(t(1),...,t(k))\in (\sum^{*})^{k}$.  
\end{definition}  

Regardless of how children manage to infer their infinite languages from finite data sets and whether or not direct negative evidence plays any (significant) role in the natural language acquisition process, it is a fact that they acquire languages which are largely consistent with the data sets they are exposed to. Any general computational learning theory should then take into consideration this specific cognitive instatiation of learning. This motivates the following notion of learnability based on texts.

\begin{definition} \label{exact learning} Let $L_{T}$ be a target language, $t$ be a text for $L_{T}$, and $\mathcal{A}$ be a learning algorithm. We say $\mathcal{A}$ exactly learns a target language $L_{T}$ on the text $t$ if there is a natural number $n_{0}$ such that $\mathcal{A}(t_{n})=L_{T}$ for all $n\geq n_{0}$. If $\mathcal{A}$ exactly learns $L_{T}$ on every text $t$ for $L_{T}$, we simply say that $\mathcal{A}$ exactly learns $L_{T}$. If $\mathscr{C}$ is a collection of languages, we say $\mathscr{C}$ is learnable in the exact sense if there is a learning algorithm $\mathcal{A}$ that exactly learns every language in $\mathscr{C}$.     
\end{definition}

It is often hypothesized that children lock to a  language after finite time (even if not the target). In many contexts, however, a notion of similarity between languages and approximation to a target is very useful. It is then only natural to introduce a metric on languages. This allows a more general mathematical framework for learnability which, as we shall see below, does not exclude the possibility of working with the important classical notion just discussed.  

\begin{definition} Given a metric $d$ on a subset of $\mathscr{L}$, a target language $L_{T}$ and a text $t$ for $L_{T}$, we say that a learning algorithm $\mathcal{A}$ learns $L_{T}$ on the text $t$ in the limit if
\[
\lim_{k \to \infty}d(\mathcal{A}(t_{k}),L_{T})=0.
\]
If $\mathcal{A}$ learns $L_{T}$ on all texts $t$ for $L_{T}$, we say $\mathcal{A}$ learns $L_{T}$ in the limit. If $\mathscr{C}$ is a set of languages and there is a learning algorithm $\mathcal{A}$ that learns every language in $\mathscr{C}$ in the limit, we say that  $\mathscr{C}$ is learnable in the limit. 
\end{definition} 

\begin{remark}\label{exactmetrics} If we let $d$ be the zero-one metric on $\mathscr{L}$ defined by $d(L,G)=1$ if $L\neq G$ and $d(L,G)=0$ if $L=G$, then we recover the notion of exact learning from definition \ref{exact learning}. In fact, any metric $d$ such that $\inf\{d(L,G):L\neq G\}>0$ would do the trick. We will refer to metrics satisfying this property as \textit{exact metrics}. Whenever we refer to exact learning below, the metric is then assumed to be exact.   
\end{remark}

Using the metric notation above, we now follow Niyogi's presentation (2006) of Gold's theorem, after Blum and Blum (1975). First we establish the necessary concepts and notation:

\begin{definition}
Given a data set $s=(s_{1},...,s_{n}) \in \mathcal{D}$, we define its $length$ to be $n$ and write $length(s)=n$. We also define its range to be the language $\{s_{1},...,s_{n}\}$ and use the notation $range(s)$ to denote it. By an abuse of notation, if $L$ is a language, we write $s \subseteq L$ meaning $range(s) \subseteq L$. For any two data sets $r=(r_{1},...,r_{m})$ and $s=(s_{1},...,s_{n})$ we define their concatenation $r \circ s$ to be the data set of length $m+n$ given by $(r_{1},...,r_{m},t_{1},...,t_{n})$.
\end{definition} 

The key notion in all that follows is that of a locking data set:  

\begin{definition} Given a learning algorithm $\mathcal{A}$, a metric $d$, a language $L$, and a real number $\varepsilon>0$, we define an $\varepsilon-$locking data set for $L$ in $\mathcal{A}$ (with respect to $d$) as a data set consisting of elements from $L$ which, once presented to the learner, leaves it $\varepsilon$-close to $L$ regardless of what other data from $L$ is presented next. Formally, a data set $l_{\varepsilon} \in \mathcal{D}$ is an $\varepsilon$-locking data set for $L$ in $\mathcal{A}$ if:$ \; i) \; l_{\varepsilon} \subseteq L; \; ii) \; d(\mathcal{A}(l_{\varepsilon}),L)<\varepsilon; \; iii) \; d(\mathcal{A}(l_{\varepsilon}\circ s), L)<\varepsilon, \text{for all} \; s \in \mathcal{D} \; \text{such that} \; s\subseteq L.$ 
\end{definition}  

The following result guarantees the existence of locking data sets whenever the language is learnable in the limit.

\begin{theorem} \label{blumandblumepsilon}(Niyogi's $\varepsilon$-version; after Blum and Blum) If $L$ is learnable in the limit, then for every $\varepsilon>0$ there is an $\varepsilon-$locking data set for $L$. 
\end{theorem}

For an exact metric, if we choose $\varepsilon<inf\{d(L,G):G \neq L\}$ in the above version, we recover the original theorem: 

\begin{theorem} \label{blumandblum75}(Blum and Blum, 1975) If $L$ is learned exactly by some $\mathcal{A}$, then there is a data set $l$ such that: $ \; i) \; l \subseteq L; \; ii) \; d(\mathcal{A}(l),L)=0; \; iii) \; d(\mathcal{A}(l\circ s), L)=0, \text{for all} \; s \in \mathcal{D} \; \text{such that} \; s\subseteq L.$  
\end{theorem}

Following Blum and Blum, Niyogi then provides elegant proofs of the two main results we are interested in. 

\begin{definition} Henceforth $\mathcal{F}$ denotes the set containing all finite languages and $L_{\infty}$ denotes an infinite language.
\end{definition}

\begin{theorem}(Gold, 1967) The set $\mathcal{F}\cup \{L_{\infty}\}$ is not learnable in the exact sense. 
\end{theorem}

The key point in the argument is the interplay between locking data sets and the languages defined by their range. When dealing with a non-exact metric, however, we can only guarantee, by applying Niyogi's $\varepsilon$-version instead of classical Blum and Blum, that the algorithm stays within an $\varepsilon$-ball from the infinite language $L_{\infty}$ and so there would be no contradiction in assuming that $\mathcal{A}$ learns the range of the locking data set in the limit as done in Niyogi's text. We explore this approximation and the fact that any expansion of the language associated with the range of an $\varepsilon$-locking data set is also $\varepsilon$-close to the infinite language. However, although our main result requires identification of the appropriate setting to generalize and makes more comprehensive use of Theorem \ref{blumandblumepsilon}, the main reasoning should be seen as a rather straightforward extension of the proof just discussed. 

Two central classes of grammars (machines) in the Chomsky hierarchy contain (recognize) all finite languages and many infinite languages. As a corollary to Gold's theorem, the family of languages generated by regular grammars (deterministic finite state automata) and context free grammars (non-deterministic pushdown automata) are not learnable in the exact sense. This provided mathematical justification to believe that either natural languages are not context free or learners have more a priori restrictions on the overall class of languages which can be learned. Around a decade later, Angluin provided a full characterization of learnability in the exact sense. 

\begin{theorem}(Angluin, 1980) A family of languages $\mathscr{C}$ is learnable in the exact sense if and only if for each $L \in \mathscr{C}$, there is a subset $D_{L}$ such that if $L'\in \mathscr{C}$ contains $D_{L}$, then $L'$ is not a proper subset of $L$.
\end{theorem}

Note that $\mathcal{F}\cup \{L_{\infty}\}$ does not satisfy the necessary part of this theorem and so Gold's theorem is immediately recovered. Our main result provides a necessary condition for the learnability of any family of languages in the limit for an arbitrary metric.

\section{Results:Learnability in the limit for general metrics}

To state our main result, we define the following notion:

\begin{definition} A sequence of languages  $(L_{n})_{n=1}^{\infty}$ is increasing if $L_{n}\subseteq L_{n+1}$ for every $n$. We say here that the sequence is strictly increasing if it is increasing and for every natural $n$ there is $m>n$ such that $L_{n}$ is properly contained in $L_{m}$. 
\end{definition}

We are now ready to present it.
 
\begin{theorem} \label{mainresult}
If a collection of languages $\mathscr{C}$ is learnable in the limit, then for every infinite language $L_{\infty} \in \mathscr{C}$, if $(L_{n})_{n=1}^{\infty}$ is a strictly increasing sequence of languages in $\mathscr{C}$ such that $\bigcup_{n=1}^{\infty}L_{n}=L_{\infty}$, then $d(L_{n},L_{\infty}) \to 0$ as $n \to \infty$.
\end{theorem}

\begin{proof} 
If $\mathscr{C}$ does not contain any infinite language, we have nothing to prove. Otherwise, consider an infinite language $L_{\infty} \in \mathscr{C}$ and fix $\varepsilon>0$. Let $\mathcal{A}$ be an algorithm that learns $\mathscr{C}$, so in particular it learns $L_{\infty}$. Then by Theorem \ref{blumandblumepsilon} there must be a locking data set $l_{\varepsilon}$ which makes $\mathcal{A}$ $\varepsilon$-close to $L_{\infty}$. If $\bigcup_{n=1}^{\infty}L_{n}=L_{\infty}$ and $(L_{n})_{n=1}^{\infty}$ is a strictly increasing sequence, there is a smallest $n_{0}$ such that $range(l_{\epsilon})\subset L_{n}$ for all $n\geq n_{0}$. Fix $n\geq n_{0}$ and define $k=length(l_{\varepsilon})$. Then construct a text for $L_{n}$ by setting $t_{k}=l_{\varepsilon}$ and making sure that every element of $L_{n}-range(l_{\varepsilon})$ (which is possibly an infinite set) is included at least once in the sequence $(t(m))_{m>k}^{\infty}$. Thus, for all $m>k$ it follows that
\[
d(L_{n},L_{\infty})\leq d(L_{n},\mathcal{A}(t_{m}))+d(\mathcal{A}(t_{m}),L_{\infty})<d(L_{n},\mathcal{A}(t_{m}))+ \varepsilon.
\]
Taking the limit as $m \to \infty$ then yields
\[
d(L_{n},L_{\infty})<\varepsilon.
\]
Since the argument applies to every $n\geq n_{0}$, we are finished. 
\end{proof}

\begin{remark} \label{goldrecovery}
	Let $\mathscr{C}=\mathcal{F}\cup\{L_{\infty}\}$ and consider an enumeration $s_{1},s_{2},...$ for the elements of $L_{\infty}$. Define  the sequence of finite languages $L_{n}=\{s_{i}:i\leq n\}$. It is clear that $(L_{n})_{n=1}^{\infty}$ is strictly increasing and $\bigcup_{n=1}^{\infty}L_{n}=L_{\infty}$. If we consider an exact metric $d$ as in Remark  \ref{exactmetrics}, it follows straight from the definition that $d(L_{n},L_{\infty})$ is always greater than a fixed real number, so convergence can never occur as $n$ increases. Therefore, $\mathscr{C}$ is not learnable in the exact sense, as stated by Gold's theorem.   
\end{remark}

Restricting our family to those containing all finite languages, the statement immediately becomes an equivalence.

\begin{corollary} \label{angluintypetheorem} Let $\mathscr{C}$ be any collection of languages containing $\mathcal{F}$. Then $\mathscr{C}$ is learnable in the limit if and only if for every increasing sequence of languages $(L_{n})_{n=1}^{\infty}$ in $\mathscr{C}$ such that $\bigcup_{n=1}^{\infty}L_{n}=L_{\infty} \in \mathscr{C}$, it holds that $d(L_{n},L_{\infty}) \to 0$ as $n \to \infty$.
\end{corollary}

\begin{proof}
One implication is given by the previous theorem. For the sufficient part, given an initial data set $t_{k}$ from a text $t$ of a target $L_{T} \in \mathscr{C}$, define a learner $\mathcal{A}$ which associates data sets to their ranges, that is, $\mathcal{A}(t_{k})=range(t_{k})$. If $L_{T}$ is finite, the learner clearly identifies and locks onto $L_{T}$ after finite steps. If $L_{T}$ is infinite, then $(\mathcal{A}(t_{n}))_{n\geq k}$ is a strictly increasing sequence of finite languages whose union is equal to $L_{\infty}$. Therefore, by hypothesis, $d(\mathcal{A}(t_{n}),L_{\infty})$ goes to zero as $n$ increases. 
\end{proof}
   
We saw that Theorem \ref{mainresult} recovers Gold's classical result and leads naturally to a characterization of learnability for a wide class of language families in any given metric. However, there is one aspect of the theorem we would like to emphasize: how it indicates that learnability of large language families can be associated with theoretically irrelevant metrics. We conclude by briefly elaborating on that next.       

Given an infinite language $L_{\infty}$, by cardinal arithmetic one can always write it as a countable union of countable languages, that is, $L_{\infty}=\cup_{i=1}^{\infty}G_{i}$, where each $G_{i}$ has cardinality $Card(G_{i})\leq \aleph_{0}$. Therefore, one can always construct a strictly increasing sequence by defining
\begin{equation} \label{countableuniondecomposition}
	L_{n}=\bigcup_{i=1}^{n}G_{i},
\end{equation}
which has the property that $L_{\infty}=\bigcup_{n=1}^{\infty}L_{n}$. Remark \ref{goldrecovery} provides an important particular example of such construction. Hence, if every $L_{n}$ and $L_{\infty}$ are learnable in a metric $d$, by the necessary condition of Theorem \ref{mainresult}, $d$ can be seen as simply counting, in the sense of being defined using information of set cardinality and not language structure. In other words, learnability of families large enough to include an infinite language $L_{\infty}$ and all $L_{n}$ from a countable union decomposition of $L_{\infty}$, as in (\ref{countableuniondecomposition}), can be associated to a counting metric. More precisely, if $d'$ is a metric by which every $L_{n}$ and $L_{\infty}$ are learnable in the limit, then for infinitely many $m$, $d'(L_{m},L_{\infty})<d'(L_{n},L_{\infty})$ if and only if $Card(L_{m}\cap L_{\infty})>Card(L_{n}\cap L_{\infty})$. Therefore, $d'$ preserves the order of the metric defined below for some infinite subset of $\{L_{n}:n \in \mathbb{N}\}\times\{L_{\infty}\}$:   

\begin{example} (A metric in which $\mathcal{F}\cup\{L_{\infty}\}$ is learnable by Corollary \ref{angluintypetheorem}) \\
Define	
\begin{align*}
	&d(L,L_{\infty})=d(L_{\infty},L)=1, \; \text{if} \; L \in \mathcal{F} \; \text{and} \; L\cap L_{\infty}= \emptyset. \\
	&d(L,L_{\infty})=\frac{1}{Card(L\cap L_{\infty})}, \; \text{if} \; L \in \mathcal{F} \; \text{and} \; L\cap L_{\infty}\neq \emptyset. \\
	&d(L,G)=d(L,L_{\infty})+d(G,L_{\infty}), \; \text{if} \; L,G \in \mathcal{F} \text{and} \; F\neq G. \\
	&d(L,G)=0, \; \text{if} \; L=G.
\end{align*}
\end{example}

It has been well understood for at least the past sixty years or so, in both formal and cognitive natural language studies, that counting does not play any (important) role in the acquisition process nor in representation of language structure \footnote{We remark that actively counting and frequency effects should not be confused and we refer to the former in our statement.}. This means the implications of Gold's result for the cognitive sciences remain largely valid in an extended setting.

\end{document}